\documentclass[11pt]{article}
\usepackage{amsmath,amssymb,amsthm}
\usepackage{latexsym}
\usepackage{graphicx}
\usepackage{algorithm}
\usepackage{algorithmic} 
\usepackage{authblk}
\usepackage{booktabs}
\usepackage{color}
\usepackage{relsize}
\usepackage{url}
\usepackage[colorlinks,linkcolor=blue,anchorcolor=green,citecolor=green]{hyperref}
\usepackage{caption} 
\usepackage{tikz}
\usepackage{yhmath}

\newtheorem{theorem}{Theorem}

\newtheorem{lemma}{Lemma}

\newtheorem{definition}{Definition}
\newtheorem{remark}{Remark}

\newtheorem{assumption}{Assumption}
\usepackage{color}

\newcommand{\vertiii}[1]{{\vert\kern-0.25ex\vert\kern-0.25ex\vert #1 
    \vert\kern-0.25ex\vert\kern-0.25ex\vert}}

\usepackage{tikz,pgfplots}
\pgfplotsset{compat=newest}
\pgfplotsset{plot coordinates/math parser=false,trim axis left}
\newlength\figureheight
\newlength\figurewidth

\title{Learning with Correntropy-induced Losses for Regression with Mixture of Symmetric Stable Noise}

\author[]{Yunlong Feng}
\author[]{Yiming Ying}

\affil[]{Department of Mathematics and Statistics,	State University of New York at Albany, New York, USA}

\date{}

\begin{document}
\maketitle

\begin{abstract}
\hspace{-0.5cm}In recent years, correntropy and its applications in machine learning have been drawing continuous attention owing to its merits in dealing with non-Gaussian noise and outliers. However, theoretical understanding of correntropy, especially in the learning theory context, is still limited. In this study, we investigate correntropy based regression in the presence of non-Gaussian noise or outliers within the statistical learning framework. Motivated by the practical way of generating non-Gaussian noise or outliers, we introduce mixture of symmetric stable noise, which include Gaussian noise, Cauchy noise, and their mixture as special cases, to model non-Gaussian noise or outliers. We demonstrate that under the mixture of symmetric stable noise assumption, correntropy based regression can learn the conditional mean function or the conditional median function well without resorting to the finite-variance or even the finite first-order moment condition on the noise. In particular, for the above two cases, we establish asymptotic optimal learning rates for correntropy based regression estimators that are asymptotically of type $\mathcal{O}(n^{-1})$. These results justify the effectiveness of the correntropy based regression estimators in dealing with outliers as well as non-Gaussian noise. We believe that the present study makes a step forward towards understanding correntropy based regression from a statistical learning viewpoint, and may also shed some light on robust statistical learning for regression.    
\end{abstract}

\section{Introduction and Motivation}\label{sec:intro}
Within the information-theoretic learning framework developed in \cite{principe2010information}, correntropy was proposed in \cite{santamaria2006generalized,liu2007correntropy} and serves as a similarity measure between two random variables. Given two scalar random variables U, V, the correntropy $\mathcal{V}_\sigma$ between $U$ and $V$ is defined as $\mathcal {V}_\sigma(U,V)=\mathbb{E} \mathcal{K}_\sigma(U,V)$ with $\mathcal{K}_\sigma$ a Gaussian kernel given by $\mathcal{K}_\sigma(u,v)=\exp\left\{-(u-v)^2/\sigma^2\right\}$, the scale parameter $\sigma>0$, and $(u,v)$ a realization of $(U,V)$. It is noticed in \cite{liu2007correntropy} that the correntropy $\mathcal {V}_\sigma(U,V)$ can induce a new metric between $U$ and $V$. It is argued in \cite{liu2007correntropy,principe2010information} that this new metric could be a better option in measuring the distance between $U$ and $V$ than the Euclidean metric when the random variable defined by the residual $U-V$ admits a non-Gaussian distribution which is frequently encountered in applications. During the past several years, the merits of correntropy have been verifying by numerous real-world applications across various fields, e.g., signal processing \cite{liu2007correntropy,chen2012maximum,chen2016generalized,chen2017maximum,zou2018robust}, image processing \cite{he2011robust,he2011maximum,he20122,hasanbelliu2014information,wang2014robust,wang2015robust,zhu2017correntropy,wang2017correntropy}, time series forecasting \cite{bessa2009entropy,bessa2010information,mendes2011development}, and many other machine learning tasks such as regression, classification, and clustering \cite{wang2013non,singh2014c,xu2016robust}. Noticing that most of the above mentioned problems can be interpreted from a regression viewpoint, recently some understanding towards correntropy based regression in statistical learning has been conducted in \cite{fenglearning} and \cite{feng2017statistical}, to which the present study is closely related. We, therefore, first revisit the conclusions on correntropy based regression drawn in \cite{fenglearning} and \cite{feng2017statistical}.

\subsection{Formulating Correntropy based Regression}
We start with the following frequently assumed data-generating model in nonparametric regression
\begin{align}\label{data_generating}
Y=f^\star(X)+\varepsilon,
\end{align}
where $X$ is the independent variable that takes values in a compact metric space $\mathcal{X}\subset\mathbb{R}^d$, $Y$ the dependent variable that takes value in $\mathcal{Y}=\mathbb{R}$, and $\varepsilon$ the noise variable. We assume that $\mathbb{E}(\varepsilon|X)=0$ if it exists, otherwise, we assume that ${\sf{median}}(\varepsilon|X)=0$. In regression problems, it is typical that we can only access a set of i.i.d observations $\mathbf{z}=\{(x_i,y_i)\}_{i=1}^n$ generated by \eqref{data_generating}. Our purpose in regression is to infer the unknown truth $f^\star$ while only referring to these observations.

The idea of correntropy based regression is to select the hypothesis from a hypothesis space that maximizes the empirical correntropy estimator between $\{y_i\}_{i=1}^n$ and $\{f(x_i)\}_{i=1}^n$ for any $f:\mathcal{X}\rightarrow\mathbb{R}$, which we term as the Maximum Correntropy Criterion based Regression (MCCR) \cite{fenglearning}. Recall that the following correntropy induced loss $\ell_\sigma: \mathbb{R} \rightarrow [0,+\infty)$ is defined in \cite{fenglearning}:
\begin{align}\label{loss_function}
\ell_\sigma(t)=\sigma^2\left(1-e^{-\frac{t^2}{\sigma^2}}\right),\,\,t\in\mathbb{R},
\end{align}
where $\sigma>0$ is a tuning parameter. MCCR can be formulated into the following empirical risk minimization scheme   
\begin{align}\label{LearningAlgorithmwithCIL}
f_\mathbf{z}:=\arg\min_{f\in\mathcal{H}}\frac{1}{n}\sum_{i=1}^n
\ell_\sigma(y_i-f(x_i)),
\end{align}
where $\mathcal{H}$ is a hypothesis space that is assumed to be a compact subset of $C(\mathcal{X})$.

\subsection{MCCR in Statistical Learning} 
As mentioned above, in the literature, correntropy and its applications in various fields have been investigated. However, in the statistical learning context, theoretical understanding of correntropy based regression estimators is still limited. Unlike commonly employed error metric in regression problems, the error metric induced by correntropy is non-convex and involves a scale parameter $\sigma$, which complicate the analysis. Recently, \cite{fenglearning} investigated correntropy based regression when the scale parameter $\sigma:=\sigma(n)$ goes large in correspondence to the sample size $n$, which was inspired by the studies in  \cite{hu2013learning,fan2014consistency} on empirical minimum error entropy minimization algorithms. When the scale parameter $\sigma(n)$ tends to zero, \cite{feng2017statistical} made some efforts in order to understand correntropy in regression problems and assess the performance of the correntropy based regression estimators from a statistical learning viewpoint. The main concerns in \cite{fenglearning} and \cite{feng2017statistical} are the learning performance of $f_\mathbf{z}$ when the sample size $n$ goes to infinity, where different scenarios of the noise variable $\varepsilon$ and the choices of the  $\sigma$ values were considered. Briefly, the following conclusions were drawn in the above-mentioned two studies:
\begin{itemize}
	\item By relating the scale parameter $\sigma$ to the sample size $n$ (i.e., $\sigma:=\sigma(n)$) and assuming that the noise variable $\varepsilon$ is zero-mean, with a diverging and properly chosen $\sigma$ value, $f_\mathbf{z}$ can approximate the conditional mean function $f^\star$ robustly. Convergence rates were established in the absence of light-tailed assumptions, which justifies the robustness of $f_\mathbf{z}$. Moreover, the scale parameter $\sigma$, in this case, plays a trade-off role between robustness and the approximation ability of the estimator $f_\mathbf{z}$.
	\item By relating the scale parameter $\sigma$ to the sample size $n$ and assuming a unique zero global mode of the noise $\varepsilon$, with a tending-to-zero and properly chosen $\sigma$ value, $f_\mathbf{z}$ approaches the conditional mode function $f^\star$. Note that the unique zero global mode assumption on $\varepsilon$ allows asymmetric or heavy-tailed noise, which again explains the robustness of the MCCR estimator $f_\mathbf{z}$ in this case. 
	\item With a properly chosen scale parameter $\sigma$, the correntropy based regression estimator $f_\mathbf{z}$ is shown to be equivalent to least squares regression estimator in the presence of symmetric and bounded noise. In this case, the equivalence is claimed in the following two senses: first, similar as that of the least squares regression estimator under the same noise condition, the population version of $f_\mathbf{z}$ is exactly the conditional mean function $f^\star$. Second, the convergence rates of $f_\mathbf{z}$ to the conditional mean function are comparable to that of least squares regression estimators.  
\end{itemize}
Some merits of MCCR can be observed from the above statements. For example, MCCR can learn $f^\star$ well in the absence of light-tailed noise assumptions where least squares regression estimators are not capable. On the other hand,  it also performs comparable with least squares regression estimators in the presence of bounded and symmetric noise where the latter one achieves its optimal performance. We refer to Section $6$ in \cite{feng2017statistical} for a general picture of existing understanding on correntropy based regression in statistical learning.

\subsection{Motivation and Contribution} 
The prominent advantages of MCCR estimator lie in its resistance ability to heavy-tailed noise and outliers. As stated above, the conducted theoretical assessments on MCCR estimators in \cite{fenglearning} and \cite{feng2017statistical} justify its superior performance in dealing with heavy-tailed noise. However, several fundamental problems related to MCCR estimators in statistical learning still remain unclear. For instance: 
\medskip

\noindent\textbf{Problem I: Learning performance of MCCR in the presence of Gaussian noise}. When Gaussian noise is present, least squares regression estimators are known to achieve their optimal performance and optimal learning rates of type $\mathcal{O}(n^{-1})$ have been established in the statistical learning literature, see e.g.,  \cite{wang2011optimal} and \cite{guo2013concentration}. Under the same noise assumption, asymptotic learning rates of type $\mathcal{O}(n^{-2/3})$ can be deduced by following the work in \cite{fenglearning}, which are not comparable with that of least squares regression estimators. Notice that the correntropy induced loss $\ell_\sigma$ is Lipschitz continuous and bounded on $\mathbb{R}$, and the fact that $\ell_\sigma$ approximates the least squares loss when $\sigma$ is large enough. It is natural to conjecture that optimal learning rates of MCCR estimators may be also achievable as least squares regression estimators in the presence of Gaussian noise.       
\medskip

\noindent\textbf{Problem II: Learning performance of MCCR with heavy-tailed noise}. In the presence of heavy-tailed noise with finite variance, from \cite{fenglearning} we know that asymptotic learning rates of type $\mathcal{O}(n^{-2/3})$ for MCCR can be established under moment assumptions. If the heavy-tailed noise has infinite variance or even infinite first-order moment condition (such as Cauchy noise), asymptotic learning rates of type $\mathcal{O}(n^{-2/5})$ were established in \cite{feng2017statistical} under mild assumptions. However, both of the above two types of learning rates are far from the type $\mathcal{O}(n^{-1})$, which are regarded as optimal in statistical learning.    
\medskip

\noindent\textbf{Problem III: Understanding MCCR in the presence of outliers}. When outliers are presented, how MCCR estimators learn the unknown truth function $f^\star$ still remains unclear, although empirically their superior performance in dealing with outliers has been observed. As mentioned above, this is, in fact, one of the most prominent advantages of MCCR estimators over other regression estimators. The main barrier to understanding MCCR in the presence of outliers lies in the modeling of outliers in analysis. This is because for the time being there exists no distribution independent definition of outlier and more frequently, outliers are defined in association with concrete distributions, see e.g., \cite{hawkins1980identification,rousseeuw2005robust,aggarwal2015outlier}. 

The present study aims to address the above three concerns on correntropy based regression, especially the concern of understanding MCCR in the presence of outliers. We start with the following motivating observation: a very frequently employed technique of generating outliers in robust statistics \cite{tukey1960survey,huber1964robust,jurevckova2005robust,huber2009robust,hampel2011robust}, machine learning \cite{scholkopf2001learning,han2011data}, as well as many engineering applications \cite{kassam1985robust,kozick2000maximum} is as follows 
\begin{align}\label{noise_mixture}
\varepsilon\thicksim \lambda_1\mathcal{N}(\mu_1,\sigma_1^2)+\lambda_2\mathcal{N}(\mu_2,\sigma_2^2),
\end{align}
where $\lambda_1+\lambda_2=1$, $\lambda_1\gg\lambda_2$, $\sigma_1^2\ll \sigma_2^2$, and $\mathcal{N}(\mu_1,\sigma_1^2)$, $\mathcal{N}(\mu_2,\sigma_2^2)$ are two Gaussian distributions with mean $\mu_1$, $\mu_2$ and variance $\sigma_1^2$, $\sigma_2^2$, respectively. In \eqref{noise_mixture}, $\mathcal{N}(\mu_1,\sigma_1^2)$ is usually considered as background noise while $\mathcal{N}(\mu_2,\sigma_2^2)$ is regarded as the contaminating noise that generates outliers since $\sigma_2^2$ is far larger than $\sigma_1^2$. In some cases, other distributions that have heavier tails than Gaussian (such as Cauchy noise) may be also employed in \eqref{noise_mixture} as contaminating noise. On the other hand, we notice that both Gaussian noise and Cauchy noise belong to the type of symmetric stable noise. These observations remind us to impose the mixture of symmetric stable noise assumption on $\varepsilon$ and study the performance of MCCR in this case. In fact, as we shall see later, mixture of symmetric stable distributions have been frequently employed in many engineering applications to model impulsive noise. Another nice property of mixture of symmetric stable noise lies in that it can approximate the distribution of any noise arbitrarily well. 

With the introduction of mixture of symmetric stable distributions in modeling heavy-tailed noise or outliers, in this paper, we make a step forward in understanding correntropy based regression in statistical learning. More detailed speaking, concerning the study of correntropy based regression estimators, in this work, we make the following contributions:
\begin{itemize}
	\item We introduce the mixture of symmetric stable distributions to model the noise $\varepsilon$. The family of mixture of symmetric stable noise includes the Gaussian noise, the mixture Gaussian noise, the Cauchy noise, and many other kinds of mixture noise, and so is capable of modeling heavy-tailed noise and outliers. We notice that within the statistical learning framework, we make some first attempts in modeling outliers via mixture of symmetric stable distributions.  
	\item Under the mixture of symmetric stable noise assumption, we demonstrate that MCCR estimators can learn the unknown truth function $f^\star$ in an unbiased way in that the population version of $f_\mathbf{z}$ is exactly $f^\star$. Recall that $f^\star$ is the conditional mean function or the conditional median function, and the mixture of symmetric stable noise consists of a large family of noise from light-tailed to heavy-tailed. This indicates that MCCR could be employed to learn $f^\star$ after seeing enough observations without resorting to the sub-Gaussianity of the noise. 
	\item We establish asymptotic learning rates of type $\mathcal{O}(n^{-1})$ which are comparable with those of least squares regression estimators under the sub-Gaussianity noise assumption. As stated above, the mixture of symmetric stable noise include Gaussian noise and Cauchy noise as two special cases, and can be used to model outliers. Therefore, the present study provides direct answers to the three problems stated above. In fact, establishing almost sure convergence rates of type $\mathcal{O}(n^{-1})$ in learning theory without appealing to finite variance assumption of the noise may be of independent interest. 
\end{itemize}

The rest of this paper is organized as follows. In Section \ref{sec::mixture_stable_noise}, we provide the definitions of symmetric stable distributions and mixture of symmetric stable distributions and introduce some of their applications. Section \ref{sec::MCCR_with_SaS} is concerned with the assessments of correntropy based regression in the presence of mixture of symmetric stable noise. The performance of MCCR, in this case, will be studied in this section, and results on learning rates of MCCR estimators will be presented here. We will also give some comments on the obtained learning rates and the MCCR estimator in this section. The paper is concluded in Section \ref{sec::conclusion}.   

\section{Mixture of Symmetric Stable Distributions and Its Applications}\label{sec::mixture_stable_noise}
In this section, we introduce the mixture of symmetric stable distributions and its applications.  To this end, we shall first introduce the symmetric stable distribution.

\begin{definition}[Symmetric Stable Distribution \cite{samorodnitsky1994stable}]\label{definition-symmetric_stable}
	A univariate distribution function is symmetric stable if its characteristic function takes the following form
	\begin{align*}
	\phi(t)=\exp\big\{\mathrm{i}\mu t-\gamma|t|^\alpha\big\},\quad \hbox{for any}\quad t\in\mathbb{R},
	\end{align*}
	where $-\infty<\mu<\infty$, $\gamma>0$, $0<\alpha\leq 2$, and $\mathrm{i}$ is the imaginary unit.
\end{definition}

More precisely, the symmetric stable distribution defined in Definition \ref{definition-symmetric_stable} is said to be $\alpha$-stable and symmetric about the location $\mu$. As shown in Definition \ref{definition-symmetric_stable}, a symmetric stable distribution has three parameters, namely, the location parameter $\mu$, the scale parameter $\gamma$, and the characteristic exponent $\alpha$. The characteristic exponent $\alpha$ is a shape parameter and measures the thickness of the tails of the density function. Two typical examples of symmetric stable distributions are Gaussian distribution ($\alpha=2$) and Cauchy distribution ($\alpha=1$). A symmetric stable distribution with $0<\alpha<2$ only admits absolute moments of order less than $\alpha$. Therefore, all symmetric stable distributions do not have finite variance except for the Gaussian distribution. For more properties of symmetric stable distributions, we refer to \cite{fama1968some,miller1978properties,samorodnitsky1994stable}.

When a univariate distribution $P$ consists of different components with each of which a symmetric stable distribution and can be expressed as a convex combination of these components, it is called a mixture of symmetric stable distributions \cite{mclachlan1988mixture}.  

\begin{definition}[Mixture of Symmetric Stable Distributions]\label{mixture_definition-symmetric_stable}
A univariate distribution $P$ with density $p$ is a mixture of symmetric stable distributions if it is a convex combination of symmetric stable distributions $\{P_i\}_{i=1}^K$ with density function $\{p_i\}_{i=1}^K$ and $K$ a positive integer, i.e., there exists $\lambda_1,\cdots,\lambda_K$ with $\lambda_i> 0$ for $i=1,\ldots,K$, and $\sum_{i=1}^K\lambda_i=1$, such that
	\begin{align*}
	P(t)=\sum_{i=1}^K \lambda_i P_i(t), \quad \hbox{and} \quad p(t)=\sum_{i=1}^K \lambda_i p_i(t), \quad\hbox{for any}\quad t\in\mathbb{R}.
	\end{align*} 
\end{definition}

In Definition \ref{mixture_definition-symmetric_stable}, $\lambda_1,\ldots,\lambda_K$ are called the mixing weights and $p_1,\ldots,p_K$ are component densities. It is obvious that when $K=1$, a mixture of symmetric stable distributions is reduced to a symmetric stable distribution. In particular, if $p_1,\ldots,p_K$ are normal densities, then $p$ is a mixture of Gaussian. A nice property of the mixture of Gaussian density is that it can approximate any density function to arbitrary accuracy with suitable choice of parameters and enough components $K$ \cite{titterington1985statistical,mclachlan1988mixture}. 

Symmetric stable distributions have been drawing continuous attention in the statistics literature \cite{fama1968some,fama1971parameter,dumouchel1973stable,miller1978properties,chen2004behavior}. The mixture of symmetric stable distributions, which includes the mixture of Gaussian and symmetric stable distributions as special cases, has been extensively applied into many applications. As mentioned above, in robust statistics, it has been employed to mimic perturbed or heavy-tailed distributions, see e.g., \cite{huber2009robust}. In many engineering applications, especially applications in the field of signal processing, image processing, and wireless communications, it has been frequently applied to model impulsive noise \cite{ shao1993signal,ambike1994detection,nikias1995signal,ilow1998analytic,kuruoglu1998near,eldar2001finite,kosko2001robust,brcich2005stability,lombardi2006line,	souryal2008soft,rajan2010diversity,park2011maximin} or outliers \cite{barnett1994outliers,aggarwal2015outlier}.

\section{MCCR with Mixture of Symmetric Stable Noise}\label{sec::MCCR_with_SaS}
The noise is mixture of symmetric stable noise if its distribution is a mixture of symmetric stable distributions. As stated in the above section, it can be employed to model non-Gaussian noise and outliers. In this section, we study MCCR from a statistical learning viewpoint in the presence of mixture of symmetric stable noise $\varepsilon$. We start with the introduction of several notations and assumptions.  

\subsection{Notations and Assumptions}
We denote the unknown probability distribution over $\mathcal{X}\times\mathcal{Y}$ as $\rho$ and $\rho_{X}$ as the marginal distribution of $\rho$ over $\mathcal{X}$. For any $f\in\mathcal{H}$, the empirical error in \eqref{LearningAlgorithmwithCIL} is denoted as $\mathcal{E}_\mathbf{z}^\sigma(f)$, that is,
\begin{align*}
\mathcal{E}_\mathbf{z}^\sigma(f)=\frac{1}{n}\sum_{i=1}^n \ell_\sigma(y_i-f(x_i)),
\end{align*}
and its population version $\mathcal{E}^\sigma(f)$ is defined as
\begin{align*}
\mathcal{E}^\sigma(f)=\int_{\mathcal{X}\times\mathcal{Y}}\ell_\sigma(y-f(x))\mathrm{d}\rho.
\end{align*}   
The distance between $f$ and $f^\star$ in $L_{\rho_X}^2$ is denoted as $\|f-f^\star\|_\rho^2$. Besides, for any two quantities $a,b$, we denote $a\lesssim b$ if there exists a positive constant $c$ such that $a\leq cb$.

\begin{assumption}[Mixture of Symmetric Stable Noise]\label{noise_assumption}
	The distribution of the noise $\varepsilon$ is a mixture of symmetric stable distributions with location parameter $0$, i.e., the density $p_\mathsmaller{\varepsilon,x}$ of the noise variable $\varepsilon$ for any $x\in\mathcal{X}$ takes the following form
	\begin{align*}
	p_\mathsmaller{\varepsilon,x}(t)=\sum_{i=1}^K\lambda_i p_\mathsmaller{\varepsilon,x,i}(t),\quad\hbox{for any}\quad t\in\mathbb{R},
	\end{align*} 
	where $K$ is a positive integer, $\lambda_i>0$ for $i=1,\ldots,K$, $\sum_{i=1}^K\lambda_i=1$, and $p_\mathsmaller{\varepsilon,x,i}$ is the density function of the symmetric stable distribution $P_\mathsmaller{\varepsilon,x,i}$ that is centered around $0$ for $i=1,\ldots,K$.  
\end{assumption}

The second assumption is on the complexity of $\mathcal{H}$ in terms of the $\ell^2$-empirical covering number $\mathcal{N}_2(\mathcal{H},\eta)$, see e.g., \cite{wu2007multi,shi2011concentration,guo2013concentration}, which is defined as follows.

\begin{definition} 
	Let ${\bf
		x}=\{x_1,x_2,\ldots,x_n\} \subset \mathcal{X}^n$. 
	The $\ell^2$-empirical covering number of the hypothesis space $\mathcal{H}$, which is denoted as $\mathcal{N}_2\left(\mathcal{H}, \eta\right)$ with radius $\eta>0$, is defined by
	\begin{align*}
	\mathcal{N}_2\left(\mathcal{H}, \eta\right):=&\sup_{n \in \mathbb{N}}\sup_{\mathbf{x}\in \mathcal{X}^n}\inf\left\{\ell\in\mathbb{N}:\exists \{f_i\}_{i=1}^\ell\subset\mathcal{H}  \,\hbox{such that for each}\,  f\in\mathcal{H},\hbox{there exists some}\right.\\
	&i\in\{1,2,\ldots,\ell\} \,\, \hbox{with}\,\,\,  \frac{1}{n}\sum_{j=1}^n |f(x_j)-f_i(x_j)|^2\leq \eta^2\Big\}.
	\end{align*}
\end{definition}

\begin{assumption}[Complexity Assumption]\label{complexity_assumption} 
	There exist positive constants $0<s<2$ and $c$ such that
	$$\log\mathcal{N}_2(\mathcal{H},\eta)\leq c \eta^{-s},\,\, \forall\,\,
	\eta>0.$$
\end{assumption}

Throughout this paper, we also assume that there exists a positive constant $M$ such that $\sup_{f\in\mathcal{H}}\|f\|_\infty\leq M$, and $\|f^\star\|_\infty\leq M$.

\subsection{Unbiasedness of MCCR with Mixture of Symmetric Stable Noise}
In the presence of mixture of symmetric stable noise, in this part, we will show that MCCR can learn $f^\star$ in an unbiased way. This is stated in the sense of the following theorem, which is established by applying techniques proposed in \cite{fan2014consistency}.  

\begin{theorem}\label{equivalent_theorem}
	Suppose that Assumption \ref{noise_assumption} holds and $f^\star\in\mathcal{H}$. Then we have
	\begin{align*}
	f^\star=\arg\min_{f\in\mathcal{H}}\mathcal{E}^\sigma(f),
	\end{align*}
	and for any $f\in\mathcal{H}$, it holds that
	\begin{align*}
	c_\mathsmaller{\sigma,\gamma,\alpha} \|f-f^\star\|_\rho^2 \leq \mathcal{E}^\sigma(f)-\mathcal{E}^\sigma(f^\star)\leq  \|f-f^\star\|_\rho^2,
	\end{align*}
	where $c_\mathsmaller{\sigma,\gamma,\alpha}$ is a positive constant that will be given explicitly in the proof.  
\end{theorem}

\begin{proof}
	From the definitions of the notions, we know that
	\begin{align*}
	\mathcal{E}^\sigma(f)-\mathcal{E}^\sigma(f^\star)=  \sigma^2\int_{\mathcal{X}}[F_x(f(x)-f^\star(x))-F_x(0)]\mathrm{d}\rho_X(x),
	\end{align*}
	where $F_x:\mathbb{R}\rightarrow \mathbb{R}$ is denoted as
	\begin{align*}
	F_x(u):=1-\int_{-\infty}^{+\infty}\exp\left\{-\frac{(t-u)^2}{\sigma^2}
	\right\}p_\mathsmaller{\varepsilon,x}(t){\mathrm{d}}t,\,\,x\in\mathcal{X}.
	\end{align*}
	From the Taylor's theorem, we know that
	\begin{align*}
	F_x(f(x)-f^\star(x))-F_x(0)= F_x^\prime(0)(f(x)-f^\star(x))+\frac{F_x^{\prime\prime}(\zeta_x)}{2}(f(x)-f^\star(x))^2,
	\end{align*} 
	where for any $x\in\mathcal{X}$, $0<\zeta_x<f(x)-f^\star(x)$. Due to the symmetry assumption of the noise, for any $x\in\mathcal{X}$, we have
	\begin{align*}
	F_x^\prime(0)=-2\int_{-\infty}^{+\infty}\exp\left(-\frac{t^2}{\sigma^2}
	\right)\left(\frac{t}{\sigma^2}\right)p_\mathsmaller{\varepsilon,x}(t){\mathrm{d}}t=0,
	\end{align*}
	and
	\begin{align*}
	F_x^{\prime\prime}(\zeta_x)=2\int_{-\infty}^{+\infty}\exp\left\{-\frac{(t-\zeta_x)^2}{\sigma^2}
	\right\}\left(\frac{\sigma^2-2(t-\zeta_x)^2}{\sigma^4}\right)p_\mathsmaller{\varepsilon,x}(t){\mathrm{d}}t,\,\,x\in\mathcal{X}.
	\end{align*}
	It is obvious that for any $x\in\mathcal{X}$, the following inequality 
	\begin{align*}
	F_x^{\prime\prime}(u)\leq \frac{2}{\sigma^2}
	\end{align*} 
	holds uniformly for $0<u<f(x)-f^\star(x)$. Therefore, we have
	\begin{align}\label{upper_bound}
	\begin{split}
	\mathcal{E}^\sigma(f)-\mathcal{E}^\sigma(f^\star)&=  \sigma^2\int_{\mathcal{X}}[F_x(f(x)-f^\star(x))-F_x(0)]\mathrm{d}\rho_X(x)\\
	&=\frac{\sigma^2}{2}\int_\mathcal{X}F_x^{\prime\prime}(\zeta_x)(f(x)-f^\star(x))^2\mathrm{d}\rho_X(x)\\
	&\leq \int_{\mathcal{X}}(f(x)-f^\star(x))^2\mathrm{d}\rho_X(x).
	\end{split}
	\end{align}
	On the other hand, with simple computations, we have
	\begin{align*}
	\mathcal{E}^\sigma(f)-\mathcal{E}^\sigma(f^\star)&=\sigma^2\int_{\mathcal{X}}\int_{-\infty}^{+\infty} \left[\exp\left(-\frac{t^2}{\sigma^2}\right)-\exp\left(-\frac{(t-[f(x)-f^\star(x)])^2}{\sigma^2}\right)\right]p_\mathsmaller{\varepsilon,x}(t){\mathrm{d}}t\mathrm{d}\rho_X(x)\\
	&=\sigma^2\int_{\mathcal{X}}\int_{-\infty}^{+\infty} \left[\exp\left(-\frac{t^2}{\sigma^2}\right)-\exp\left(-\frac{(t-[f(x)-f^\star(x)])^2}{\sigma^2}\right)\right]p_\mathsmaller{\varepsilon,x}(t){\mathrm{d}}t\mathrm{d}\rho_X(x)\\
	&=\sigma^2\int_{\mathcal{X}}\int_{-\infty}^{+\infty} \left[\exp\left(-\frac{t^2}{\sigma^2}\right)-\exp\left(-\frac{(t-u_x)^2}{\sigma^2}\right)\right]p_\mathsmaller{\varepsilon,x}(t){\mathrm{d}}t\mathrm{d}\rho_X(x)\\
	&=\sigma^2\int_{\mathcal{X}}\int_{-\infty}^{+\infty} \left[\exp\left(-\frac{(t+u_x)^2}{\sigma^2}\right)-\exp\left(-\frac{t^2}{\sigma^2}\right)\right]p_\mathsmaller{\varepsilon,x}(t){\mathrm{d}}t\mathrm{d}\rho_X(x),
	\end{align*} 
	where for any $x\in\mathcal{X}$, $u_x=f(x)-f^\star(x)$. From  Assumption \ref{noise_assumption} on the noise and recalling the linearity property of the Fourier transform, we have
	\begin{align*}
	\widehat{p_\mathsmaller{{\epsilon,x}}}(\xi)=\sum_{i=1}^K \lambda_i \widehat{p_\mathsmaller{{\epsilon,x,i}}}(\xi),
	\end{align*}
	where $\widehat{p_\mathsmaller{{\epsilon,x}}}$ is the Fourier transform of $p_\mathsmaller{{\epsilon,x}}$, and $\widehat{p_\mathsmaller{{\epsilon,x,i}}}$ is the Fourier transform of $p_\mathsmaller{{\epsilon,x,i}}$, $i=1,\ldots,K$. Moreover, for $i=1,\ldots,K$, since $P_\mathsmaller{{\epsilon,x,i}}$ is a symmetric stable distribution with the location parameter $0$, we know that there exist  $\gamma_i>0$ and $0<\alpha_i\leq 2$ such that
	\begin{align*}
	\widehat{p_\mathsmaller{{\epsilon,x,i}}}(\xi)=e^{-\gamma_i|\xi|^{\alpha_i}}.
	\end{align*} 
	Applying the Planchel formula, we obtain
	\begin{align*}
	\mathcal{E}^\sigma(f)-\mathcal{E}^\sigma(f^\star)&=\frac{\sigma^3}{2\sqrt{\pi}}\int_\mathcal{X}\int_{-\infty}^{+\infty}\exp\left(-\frac{\sigma^2\xi^2}{4}\right)\widehat{p_\mathsmaller{{\epsilon,x}}}(\xi)\left[1-\bold{e}^{\mathrm{i} \xi u_x }\right]\mathrm{d}\xi \mathrm{d}\rho_X(x)\\
	&=\frac{\sigma^3}{\sqrt{\pi}}\sum_{i=1}^K\lambda_i\int_\mathcal{X}\int_{-\infty}^{+\infty}\exp\left(-\frac{\sigma^2\xi^2}{4}\right)\widehat{p_\mathsmaller{{\epsilon,x,i}}}(\xi)\sin^2\left(\frac{\xi(f(x)-f^\star(x))}{2}\right)\mathrm{d}\xi \mathrm{d}\rho_X(x)\\
	&= \frac{\sigma^3}{\sqrt{\pi}}\int_\mathcal{X}\sum_{i=1}^K\lambda_i\int_{-\infty}^{+\infty}\exp\left(-\frac{\sigma^2\xi^2}{4}-\gamma_i|\xi|^{\alpha_i}\right)\sin^2\left(\frac{\xi(f(x)-f^\star(x))}{2}\right)\mathrm{d}\xi \mathrm{d}\rho_X(x),
	\end{align*}
	where the second equality is due to the fact that $\mathcal{E}^\sigma(f)-\mathcal{E}^\sigma(f^\star)$ is real for any $f\in\mathcal{H}$. For any $x\in\mathcal{X}$, $|u_x|=|f(x)-f^\star(x)|\leq 2M$. When $|\xi|\leq \frac{\pi}{2M}$, from Jordan's inequality, it holds that
	\begin{align*}
	\sin^2\left(\frac{\xi(f(x)-f^\star(x))}{2}\right)\geq \frac{2\xi^2(f(x)-f^\star(x))^2}{\pi^2}.
	\end{align*} 
	As a result, we come to the following conclusion
	\begin{align}\label{lower_bound}
	\begin{split}
	\mathcal{E}^\sigma(f)-\mathcal{E}^\sigma(f^\star)&\geq \frac{2\sigma^3}{\pi^{5/2}}\int_{\mathcal{X}}\sum_{i=1}^K\lambda_i\int_{-\frac{\pi}{2M}}^{\frac{\pi}{2M}}\xi^2\exp\left(-\frac{\sigma^2\xi^2}{4}-\gamma_i|\xi|^{\alpha_i}\right)(f(x)-f^\star(x))^2\mathrm{d}\xi \mathrm{d}\rho_X(x)\\
	&=c_\mathsmaller{\sigma,\gamma,\alpha}\int_{\mathcal{X}}(f(x)-f^\star(x))^2\mathrm{d}\rho_X(x),
	\end{split}
	\end{align}
	where  
	\begin{align}\label{constant_c_1}
	c_\mathsmaller{\sigma,\gamma,\alpha}=\frac{2\sigma^3}{\pi^{5/2}} \sum_{i=1}^K\lambda_i\int_{-\frac{\pi}{2M}}^{\frac{\pi}{2M}}\xi^2\exp\left(-\frac{\sigma^2\xi^2}{4}-\gamma_i|\xi|^{\alpha_i}\right)\mathrm{d}\xi.
	\end{align}
	The positiveness of $c_\mathsmaller{\sigma,\gamma,\alpha}$ implies that for any $f\in\mathcal{H}$, we have $\mathcal{E}^\sigma(f)\geq \mathcal{E}^\sigma(f^\star)$. That is,
	\begin{align*}
	f^\star=\arg\min_{f\in\mathcal{H}}\mathcal{E}^\sigma(f).
	\end{align*}
	To prove the second assertion, we combine inequalities \eqref{upper_bound} and \eqref{lower_bound}, and obtain
	\begin{align*}
	c_\mathsmaller{\sigma,\gamma,\alpha} \|f-f^\star\|_\rho^2 \leq \mathcal{E}^\sigma(f)-\mathcal{E}^\sigma(f^\star)\leq  \|f-f^\star\|_\rho^2,
	\end{align*}
	where $c_\mathsmaller{\sigma,\gamma,\alpha}$ is a positive constant given in \eqref{constant_c_1}. This completes the proof of Theorem \ref{equivalent_theorem}.
\end{proof}

Theorem \ref{equivalent_theorem} states that in the presence of mixture of symmetric stable noise, the population version of the MCCR estimator $f_\mathbf{z}$ is exactly the underlying unknown truth function $f^\star$ as long as $f^\star$ belongs to $\mathcal{H}$. Therefore, in this sense, $f_\mathbf{z}$ can be regarded as an unbiased estimator of $f^\star$. Another implication of Theorem \ref{equivalent_theorem} is that under the mixture of symmetric stable noise assumption, the excess risk of MCCR can be upper and lower bounded by the $L^2_{\rho_X}$-distance between the MCCR estimator $f_\mathbf{z}$ and the unknown truth $f^\star$. As we shall see later, this leads to fast convergence rates of the MCCR estimator $f_\mathbf{z}$ to $f^\star$.

\subsection{Performance of MCCR with Mixture of Symmetric Stable Noise}
We are now in a position to evaluate the learning performance of MCCR in the presence of mixture of symmetric stable noise by establishing convergence rates of $\|f_\mathbf{z}-f^\star\|_\rho^2$.

\begin{theorem}\label{performance_MCCR}
	Suppose that Assumption \ref{noise_assumption} and Complexity Assumption with $s>0$ hold. Let $f_\mathbf{z}$ be produced by \eqref{LearningAlgorithmwithCIL} and $f^\star\in\mathcal{H}$. For any $0<\delta<1$, with confidence $1-\delta$, it holds that
	\begin{align*}
	\|f_\mathbf{z}-f^\star\|_\rho^2\lesssim \log(1/\delta)n^{-\frac{2}{2+s}}.
	\end{align*} 
\end{theorem}

When functions in $\mathcal{H}$ are sufficiently smooth, the index $s$ could be arbitrarily small. Therefore, it is immediate to see that the convergence rates established in Theorem \ref{performance_MCCR} are asymptotically of type $\mathcal{O}(n^{-1})$. Recall that in Theorem \ref{performance_MCCR}, the noise $\varepsilon$ is only assumed to be a mixture of symmetric stable noise which include the mixture Gaussian and the Cauchy noise, and can be applied to model outliers. It is interesting to see that in this case the MCCR estimator $f_\mathbf{z}$ can learn the conditional mean function or the conditional median function $f^\star$ well. This, in fact, explains the merits of MCCR in dealing with heavy-tailed noise or outliers. Moreover, as far as we are aware, within the statistical learning framework, we present some first results on the optimal convergence rates of regression estimator without imposing finite-variance or even finite first-order moment conditions on the noise. 

To prove Theorem \ref{performance_MCCR}, we need the following lemma established in \cite{wu2007multi}. 
\begin{lemma}\label{empiricalconcen}
	Let $\mathcal{F}$ be a class of measurable functions on
	$\mathcal{Z}$. Assume that there are constants $B,c>0$ and
	$\theta\in[0,1]$ such that $\|f\|_{\infty}\leq B$ and
	$\mathbb{E}f^2\leq c(\mathbb{E}f)^{\theta}$ for every $f\in
	\mathcal{F}$. If for some $a>0$ and $s\in (0,2)$,
	\begin{equation*} 
	\log\mathcal{N}_2\left(\mathcal{F},\eta\right)\leq a\eta^{-s},
	\qquad \forall\,\eta >0,
	\end{equation*}
	then there exists a constant $\alpha_p$ depending only on $p$ such
	that for any $t>0$, with probability at least $1-e^{-t}$, there
	holds
	\begin{equation*}
	\mathbb{E}f-\frac{1}{m}\sum_{i=1}^mf(z_i)\leq
	\frac{1}{2}\gamma^{1-\theta}\left(\mathbb{E}f\right)^{\theta}+\alpha_p\gamma
	+2\left(\frac{ct}{m}\right)^{\frac{1}{2-\theta}}+\frac{18Bt}{m},
	\qquad \forall\, f\in \mathcal{F},
	\end{equation*}
	where
	$$\gamma:=\max\left\{c^{\frac{2-s}{4-2\theta+s\theta}}\left(\frac{a}{m}\right)
	^{\frac{2}{4-2\theta+s\theta}},B^{\frac{2-s}{2+s}}\left(\frac{a}{m}\right)
	^{\frac{2}{2+s}}\right\}.$$
\end{lemma}

\begin{proof}[Proof of Theorem \ref{performance_MCCR}]
To prove Theorem \ref{performance_MCCR}, we apply Lemma \ref{empiricalconcen} to the function set $\mathcal{F}_\mathcal{H}$ defined below
\begin{align*}
\mathcal{F}_\mathcal{H}=\left\{g \,\,\Big |\,\, g(z)=-\sigma^2\exp\left\{-(y-f(x))^2/\sigma^2\right\}+
\sigma^2\exp\left\{-(y-f^\star(x))^2/\sigma^2\right\},
f\in\mathcal{H},z\in\mathcal{Z}\right\}.
\end{align*}
	We first verify conditions in Lemma \ref{empiricalconcen}. From the definition of $\mathcal{F}_\mathcal{H}$, for any $g\in\mathcal{F}_\mathcal{H}$, we have
	\begin{align*}
	\|g\|_\infty\leq \sigma^2+\sigma^2=2\sigma^2,
	\end{align*}
	and the following Bernstein condition holds
	\begin{align}\label{Bernstein_condition}
	\begin{split}
	\mathbb{E}g^2&=\int_\mathcal{Z}\left(-\sigma^2\exp\left\{-\frac{(y-f(x))^2}{\sigma^2}\right\}+
	\sigma^2\exp\left\{-\frac{(y-f^\star(x))^2}{\sigma^2}\right\}\right)^2\mathrm{d}\rho\\
	& \lesssim \sigma^2\int_\mathcal{Z} \left((y-f(x))-(y-f^\star(x))\right)^2\mathrm{d}\rho\\
	&=\sigma^2\int_\mathcal{X}(f(x)-f^\star(x))^2\mathrm{d}\rho \lesssim \mathbb{E}g,
	\end{split}
	\end{align}
	where the first inequality is a consequence of the mean value theorem and the boundedness of $\|h^\prime\|$ with $h(t)=-\sigma^2\exp(-t^2/\sigma^2)$, $t\in\mathbb{R}$, and the second inequality is due to Theorem \ref{equivalent_theorem}. On the other hand, for any $g_1$, $g_2\in \mathcal{F}_\mathcal{H}$, there exist $f_1,\,f_2\in\mathcal{H}$ such that
	\begin{align*}
	g_1(z)=-\sigma^2\exp\left\{-(y-f_1(x))^2/\sigma^2\right\}+
	\sigma^2\exp\left\{-(y-f^\star(x))^2/\sigma^2\right\},
	\end{align*} 
	and  
	\begin{align*}
	g_2(z)=-\sigma^2\exp\left\{-(y-f_2(x))^2/\sigma^2\right\}+
	\sigma^2\exp\left\{-(y-f^\star(x))^2/\sigma^2\right\}.
	\end{align*}
	By applying the mean value theorem and noticing again the boundedness of $\|h^{\prime}\|_\infty$, we have
	\begin{align*}
	\|g_1-g_2\|_\infty\leq \sigma^2\|f_1-f_2\|_\infty.
	\end{align*}
	Under the Complexity Assumption with $0<s<2$, the following relation between the $\ell^2$-empirical covering numbers of $\mathcal{F}_\mathcal{H}$ and $\mathcal{H}$ holds
	\begin{align*}
	\log\mathcal{N}_2(\mathcal{F}_\mathcal{H},\eta)\leq
	\log\mathcal{N}_2\Big(\mathcal{H},\eta/\sigma^2\Big)\lesssim \eta^{-s}.
	\end{align*}
	Applying Lemma \ref{empiricalconcen} to the function set
	$\mathcal{F}_\mathcal{H}$, with simple computations, we come to the conclusion that for any $0<\delta<1$ with confidence $1-\delta$, there holds
	\begin{align*}
	\left[\mathcal{E}^\sigma(f)-\mathcal{E}^\sigma(f^\star)\right]-
	\left[\mathcal{E}^\sigma_\mathbf{z}(f)-\mathcal{E}^\sigma_\mathbf{z}(f^\star)\right]-\frac{1}{2}\left[\mathcal{E}^\sigma(f)-\mathcal{E}^\sigma(f^\star)\right]\lesssim \log(1/\delta)n^{-\frac{2}{2+s}}.
	\end{align*}
	Noticing that $\mathcal{E}_\mathbf{z}^\sigma(f_\mathbf{z})\leq \mathcal{E}_\mathbf{z}^\sigma(f^\star)$, we have
	\begin{align*}
	\frac{1}{2}\left[\mathcal{E}^\sigma(f_\mathbf{z})-\mathcal{E}^\sigma(f^\star)\right]\leq \left[\mathcal{E}^\sigma(f_\mathbf{z})-\mathcal{E}^\sigma(f^\star)\right]-
	\left[\mathcal{E}^\sigma_\mathbf{z}(f_\mathbf{z})-\mathcal{E}^\sigma_\mathbf{z}(f^\star)\right]-\frac{1}{2}\left[\mathcal{E}^\sigma(f_\mathbf{z})-\mathcal{E}^\sigma(f^\star)\right].
	\end{align*}
	Therefore, for any $0<\delta<1$ with confidence $1-\delta$, it holds that
	\begin{align*}
	\|f_\mathbf{z}-f^\star\|_\rho^2\lesssim \log(1/\delta)n^{-\frac{2}{2+s}}.
	\end{align*}
	This completes the proof of Theorem \ref{performance_MCCR}.
\end{proof}

\begin{remark}
	From the proof of Theorem \ref{performance_MCCR}, we see that the boundedness of the loss function $\ell_\sigma$ and the Bernstein condition \eqref{Bernstein_condition} play a crucial role in establishing fast convergence rates of $f_\mathbf{z}$. The Bernstein condition holds because of the Lipschitz continuity of the loss function $\ell_\sigma$ on $\mathbb{R}$ and the fact that the $L_{\rho_X}^2$-distance between $f_\mathbf{z}$ and $f^\star$ can be upper bounded by the excess risk $\mathcal{E}^\sigma(f_\mathbf{z})-\mathcal{E}^\sigma(f^\star)$, i.e., conclusions in Theorem \ref{equivalent_theorem}.     
\end{remark}

\subsection{Comments on MCCR with Mixture of Symmetric Stable Noise}
We now give two remarks on the performance of the MCCR estimator $f_\mathbf{z}$ in the presence of mixture of symmetric stable noise by comparing with that of the least squares estimator. 

The first remark is on the convergence rates of the two regression estimators. As shown in Theorem \ref{performance_MCCR}, in the presence of mixture of symmetric stable noise and when $f^\star\in\mathcal{H}$, $f_\mathbf{z}$ can learn the unknown truth function $f^\star$ well. The established learning rates are of   type $\mathcal{O}(n^{-\frac{2}{2+s}})$ which are optimal in the sense that they are asymptotically of type $\mathcal{O}(n^{-1})$. Moreover, they are comparable with that of least squares estimators \cite{wu2006learning,cucker2007learning}.

Our second remark is on the conditions required to established convergence rates for the two regression schemes. Recalling that for least squares regression, to establish learning theory type convergence rates, the response variable (and consequently the noise, under the data-generating model \eqref{data_generating}) is frequently assumed to be uniformly bounded \cite{cucker2007learning,steinwart2008support}, which is usually not the case in practice. In fact, even in the presence of Gaussian noise, to establish learning theory type convergence rates for least squares regression, it is much involved due to the unboundedness of the response variable, in which case many conventional learning theory arguments and tools are not applicable. Recently, some efforts have been made to relax this assumption \cite{wang2011optimal,guo2013concentration,mendelson2015learning}. As far as we are aware, convergence rates for least squares regression estimators cannot be established without resorting to the finite-variance condition. When moving our attention to correntropy based regression, as shown above, in the presence of mixture of symmetric stable noise, optimal learning rates of MCCR estimator are established. Notice that symmetric stable noise with the characteristic exponent parameter $0<\alpha<2$ has infinite variance or even first-order moment. Moreover, as stated above, it can approximate any density function arbitrarily well with properly chosen $K$ and consequently can be applied to model outliers. In this sense, our study presented here explains the capability of MCCR estimators in dealing with outliers.

\section{Simulations}
In this section, we provide simulations (1) to validate the feasibility of modeling outliers by using mixture of symmetric stable distributions and (2) to justify the robustness of MCCR to outliers by comparing with that of Huber regression estimators which are regarded as outlier robust. 

Concerning the data generating model $Y=f^\star(X)+\varepsilon$, we set the truth function $f^\star$ as the following sinc function  
\begin{align*}
f^\star(x)=\sin(\pi x)/(\pi x), x\in [-4,4],
\end{align*}  
as done in \cite{vapnik1998statistical,smola2004tutorial}. In our simulation studies, we aim to learn $f^\star$ from observations that are contaminated by outliers. In particular, the outliers are generated by mixture of symmetric stable noise as proposed in this study. We consider the following two types of noise that belong to this category:
\begin{itemize}
	\item Noise I: $\varepsilon \sim 0.9N(0,0.05^2)+0.1N(0,0.5^2)$
	\item Noise II: $\varepsilon \sim 0.9N(0,0.05^2)+0.1\hbox{Cauchy}(0,1)$
\end{itemize} 
For Noise I, it is drawn from the mixture of two Gaussian distributions where the background noise is drawn from $N(0,0.05^2)$ and the contaminating noise is drawn from $N(0,0.5^2)$ to generate outliers. For Noise II, it is drawn from the mixture of Gaussian and Cauchy distributions where the Gaussian noise $N(0,0.05^2)$ serves as background noise and outliers are generated by the contaminating noise $\hbox{Cauchy}(0,1)$, i.e., Cauchy noise with the location parameter $0$ and the scale parameter $1$.   
  
\begin{figure}[!htb]
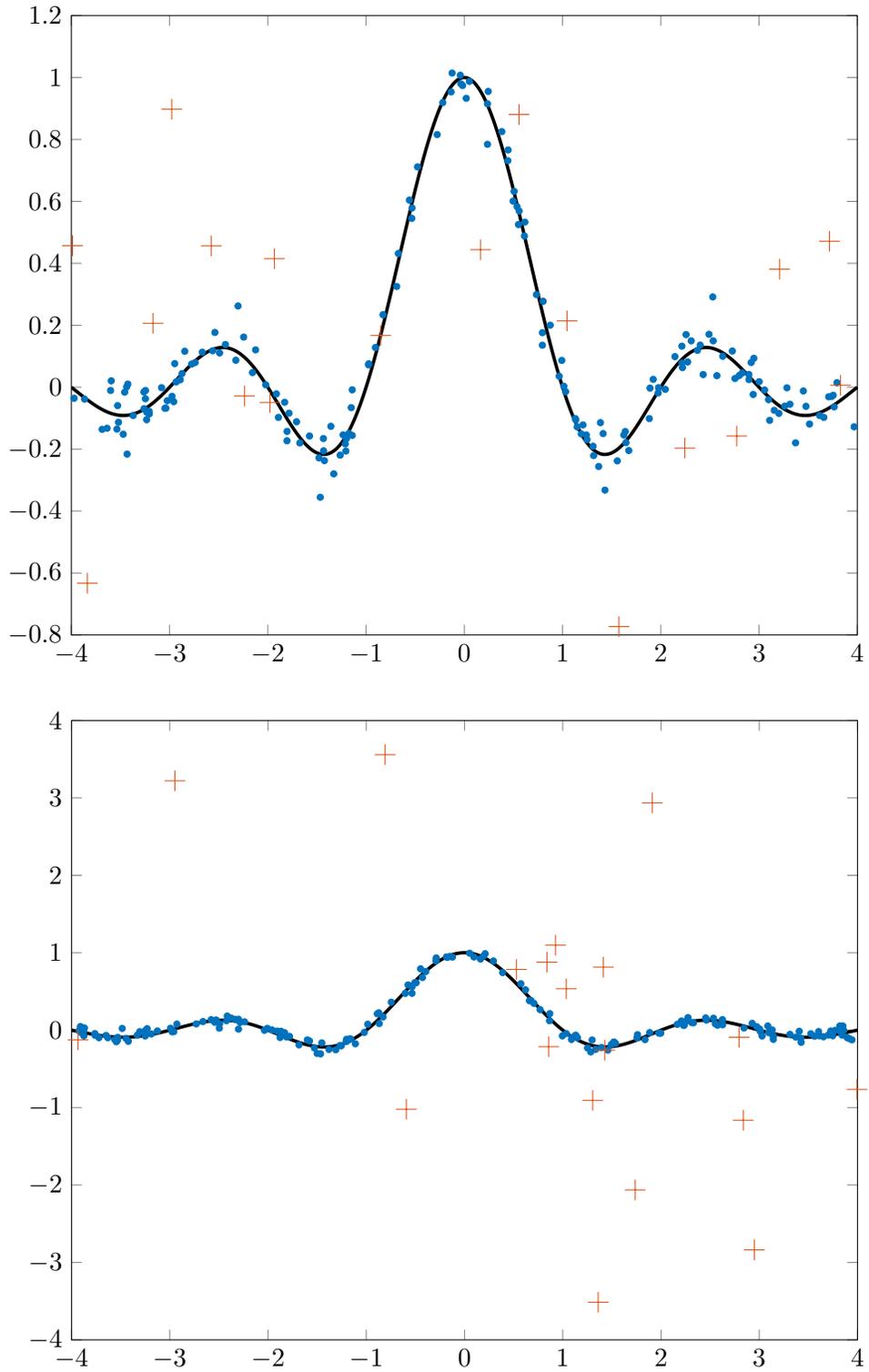

	\centering
	\begin{minipage}[b]{1\textwidth}
		\centering
		\input{outlier_Gaussian.tikz}
	\end{minipage}%\
	\\
	\bigskip
	\begin{minipage}[b]{1\textwidth}
		\centering
		\input{outlier_Cauchy.tikz}
	\end{minipage}%
	\caption{Sinc function (black solid curves) and training samples. The samples with red crosses are regarded as outliers. (top) The observations are contaminated by mixture of Gaussian noise. (bottom) The observations are contaminated by mixture of Gaussian and Cauchy noise.}
	\label{outliers_generating}
\end{figure} 

\begin{figure}[!htb]
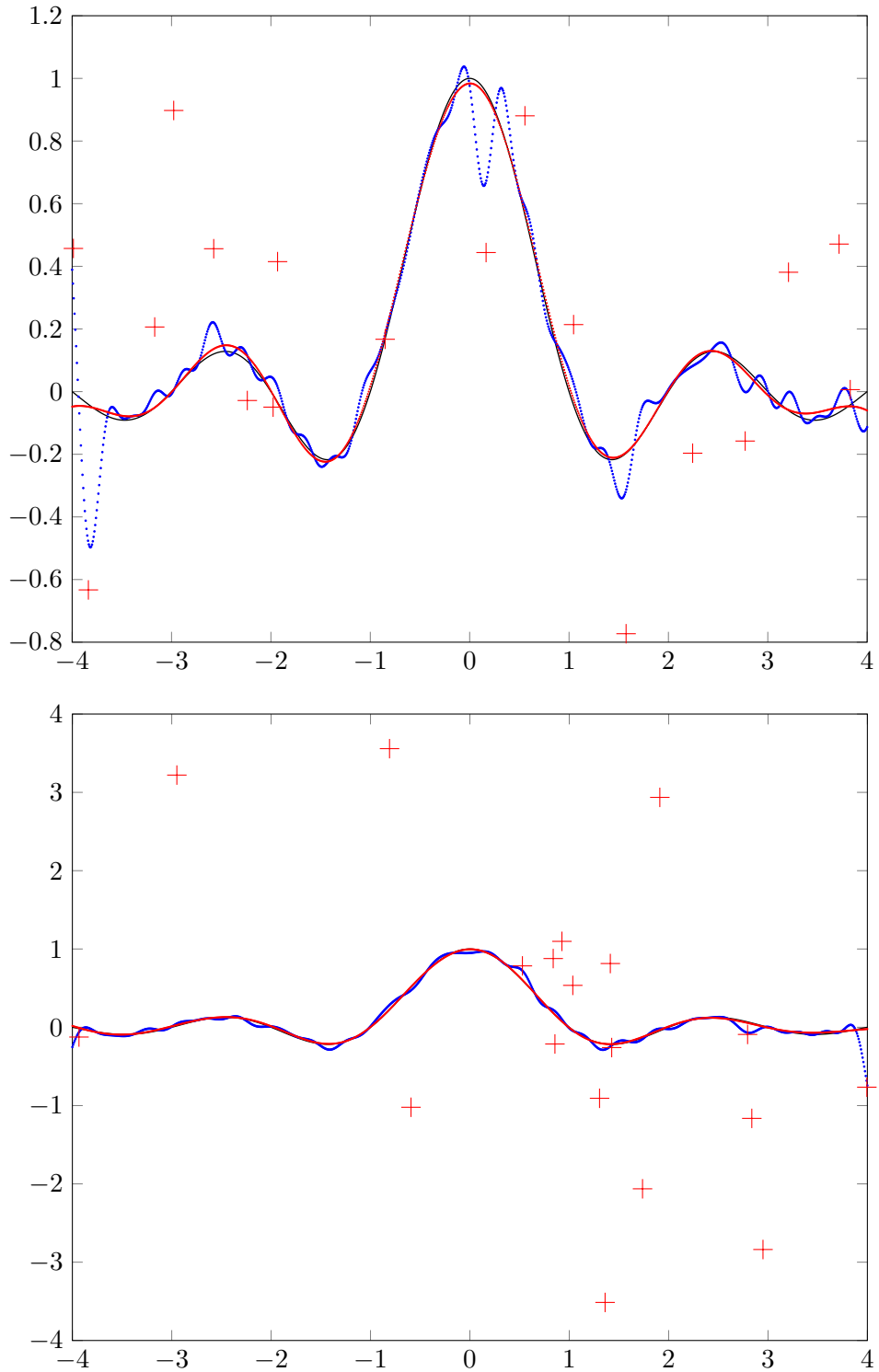

	\centering
	\begin{minipage}[b]{1\textwidth}
		\centering
		\input{fitting_Gaussian.tikz}
	\end{minipage}%\
	\\
	\medskip 
	\begin{minipage}[b]{1\textwidth}
		\centering
		\input{fitting_Cauchy.tikz}
	\end{minipage}%
	\caption{Outliers (red crosses), sinc function (black solid curves) and its estimators from MCCR and Huber regression (MCCR: red  dashed curve; Huber: blue dashed curve). (top) The observations are contaminated by mixture of Gaussian noise. (bottom) The observations are contaminated by mixture of Gaussian and Cauchy noise.}
	\label{fitting_result}
\end{figure}

We set up our experiment by following that of \cite{fenglearning}, i.e., the hypothesis space $\mathcal{H}$ is chosen as a subset of a reproducing kernel Hilbert space which is selected automatically by means of a regularized empirical risk minimization, see formula (21) in \cite{fenglearning}. A Gaussian kernel is utilized as the reproducing kernel. $200$ samples are drawn as training data and $400$ samples are drawn as test data. The bandwidth parameter, the regularization parameter, and the scale parameter in Huber's loss are tuned via a five-fold cross validation. The scale parameter $\sigma$ in the loss function $\ell_\sigma$ \eqref{loss_function} is set to $0.01$. 

Experimental results on the generation of outliers and the learned curves are plotted in Figs.\, \ref{outliers_generating} and \ref{fitting_result}. In Fig.\,\ref{outliers_generating}, the black curves stand for the curve of the truth function $f^\star$. The blue dots from the two panels stand for samples that are contaminated by the background noise of Noise I and Noise II, respectively. The red crosses are samples contaminated by contaminating noise of the two noise types, respectively, which are regarded as outliers. In Fig.\,\ref{fitting_result}, the truth curve (black solid line) as well as the curves learned from MCCR (dashed red curve) and from Huber regression (dashed blue curve) are plotted when the noise are of type I and type II, respectively. Outliers are also marked in Fig.\,\ref{fitting_result} for illustration.

From Fig. \ref{outliers_generating}, it is easy to see that outliers are indeed generated when the noise are drawn from mixture of symmetric stable distributions. According to Fig. \ref{fitting_result}, MCCR is robust to outliers and performs better than Huber regression in the presence of outliers.

\section{Conclusion}\label{sec::conclusion}
In this paper, we studied the correntropy based regression within the statistical learning framework by introducing the mixture of symmetric stable noise which  subsume Gaussian noise, Cauchy noise, and mixture of Gaussian noise. In this study, it was introduced to model heavy-tailed noise and outliers, to which the correntropy based regression estimators have been empirically verified to be resistant. In our study, we showed that the empirical risk minimization scheme based on the correntropy induced loss can learn the underlying truth function sufficiently well while allowing the noise to be the mixture of symmetric stable noise. In particular, learning theory analysis was conducted and the learning performance of MCCR with mixture of symmetric stable noise was evaluated. It is interesting to see that, in this case, asymptotically optimal learning rates of type $\mathcal{O}(n^{-1})$ can be developed, which are comparable with that of least squares regression under bounded noise assumption. These theoretical findings successfully explain the efficiency and effectiveness of correntropy based regression estimators in the presence of heavy-tailed noise or outliers.   

\section*{Acknowledgement}
The authors would like to than the referees for their constructive suggestions and comments. The work of Yiming Ying is supported by National Science Foundation (NSF) under Grant No. 1816227.

\bibliographystyle{plain}  
\bibliography{FENGBib} 
\end{document}